\documentclass{article}
\usepackage[preprint]{neurips_2019}   

\usepackage[utf8]{inputenc} 
\usepackage[T1]{fontenc}    
\usepackage{hyperref}       
\usepackage{url}            
\usepackage{booktabs}       
\usepackage{amsfonts}       
\usepackage{nicefrac}       
\usepackage{microtype}      

\usepackage[format=plain,
            labelfont={bf,it},
            textfont=it]{caption}
\usepackage{amsmath,amssymb,amsthm,thm-restate}
\usepackage[boxruled,vlined,nofillcomment]{algorithm2e}
\usepackage{graphicx}
\SetKwProg{Fn}{}{\string:}{}

\newcommand{\R}{\mathbf{R}}

\newcommand{\E}{\mathbf{E}}
\newcommand{\Var}{\mathbf{Var}}
\newcommand{\indicator}{\mathbf{1}}
\renewcommand{\mod}{\textrm{ mod }}
\newtheorem{definition}{Definition}

\newtheorem{lemma}{Lemma}
\newtheorem{corollary}{Corollary}

\newcommand\FULL{}
\renewcommand{\epsilon}{\varepsilon}
\newcommand{\splitter}[2]{\ifdefined\FULL#1\else#2\fi}

\splitter{ 
  \newcommand{\supmat}[1]{Appendix{#1}}
  \newcommand{\Supmat}[1]{Appendix{#1}}
  
  \newcommand{\supmatopt}[1]{#1}
}
{ 
  \newcommand{\supmat}[1]{the supplementary material}
  \newcommand{\Supmat}[1]{The supplementary material}
  \newcommand{\supmatopt}[1]{}
}

\newcommand{\D}{\mathcal{D}}

\newcommand{\smoothset}{\binom{\mathcal{Y}}{2m}_{\gamma\textrm{-smooth}}}

\usepackage{xcolor}

\title{Scalable and Differentially Private Distributed Aggregation
in the Shuffled Model\\ }

\author{
  Badih Ghazi\\
  Google Research\\
  \texttt{badihghazi@gmail.com} \\
  \And
  Rasmus Pagh \\
  Google Research\\ 
  \& IT University of Copenhagen\\ 
  \texttt{pagh@itu.dk}
  \And
  Ameya Velingker \\
  Google Research \\
  \texttt{ameyav@google.com}
}

\begin{document}

\maketitle

\begin{abstract}
  
  \emph{Federated learning} promises to make machine learning feasible on distributed, private datasets by implementing gradient descent using secure aggregation methods.
  The idea is to compute a global weight update without revealing the contributions of individual users.

  Current practical protocols for secure aggregation work in an ``honest but curious'' setting where a curious adversary observing all communication to and from the server cannot learn any private information assuming the server is honest and follows the protocol.
  
  A more scalable and robust primitive for privacy-preserving protocols is \emph{shuffling} of user data, so as to hide the origin of each data item.
  Highly scalable and secure protocols for shuffling, so-called \emph{mixnets}, have been proposed as a primitive for privacy-preserving analytics in the Encode-Shuffle-Analyze framework by Bittau et al., which was later analytically studied by Erlingsson et al. and Cheu et al.. The recent papers by Cheu et al., ~and Balle et al.~have given protocols for secure aggregation that achieve \emph{differential privacy} guarantees in this ``shuffled model''.
  Their protocols come at a cost, though:
  Either the expected aggregation error \emph{or} the amount of communication per user scales as a polynomial $n^{\Omega(1)}$ in the number of users $n$.

  In this paper we propose simple and more efficient protocol for aggregation in the shuffled model, where communication as well as error increases only polylogarithmically in $n$.
  Our new technique is a conceptual ``invisibility cloak'' that makes users' data almost indistinguishable from random noise while introducing \emph{zero} distortion on the sum.
\end{abstract}

\section{Introduction}

We consider the problem of privately summing $n$ numbers in the \emph{shuffled model} which is based on the Encode-Shuffle-Analyze architecture of Bittau et al.~\cite{bittau17} and was first analytically studied in \cite{erlingsson2019amplification, cheu19}. For consistency with the literature we will use the term \emph{aggregation} for the sum operation.
Consider $n$ users with data $x_1,\dots,x_n \in [0,1]$.
In the shuffled model user $i$ applies a randomized \emph{encoder} algorithm $E$ that maps $x_i$ to a multiset of $m$ messages, $E(x_i) = \{y_{i,1},\dots,y_{i,m}\} \subseteq \mathcal{Y}$, where $m$ is a parameter.
Then a trusted \emph{shuffler} $\mathcal{S}$ takes all $nm$ messages and outputs them in random order.
Finally, an \emph{analyzer} algorithm $\mathcal{A}$ maps the shuffled output $\mathcal{S}(E(x_1),\dots,E(x_n))$ to an estimate of~$\sum_i x_i$.

A protocol in the shuffled model is $(\varepsilon,\delta)$-differentially private if $\mathcal{S}(R_1(x_1),\dots,R_n(x_n))$ is $(\varepsilon,\delta)$-differentially private (see definition in Section~\ref{sec:preliminaries}), where probabilities are with respect to the random choices made in the algorithm $E$ and the shuffler $\mathcal{S}$.
The privacy claim is justified by the existence of highly scalable protocols for privately implementing the shuffling primitive~\cite{bittau17,cheu19}.

Two protocols for aggregation in the shuffled model were recently suggested by Balle et al.~\cite{balle19} and Cheu et al.~\cite{cheu19}.
We discuss these further in Section~\ref{sec:related}, but note here that all previously known protocols have either communication or error that grows as $n^{\Omega(1)}$.
This is unavoidable for single-message protocols, by the lower bound in~\cite{balle19}, but it has been unclear if such a trade-off is necessary in general.
Cheu et al.~\cite{cheu19} explicitly mention it as an open problem to investigate this question.

\subsection{Our Results}

We show that a trade-off is not necessary --- it is possible to avoid the $n^{\Omega(1)}$ factor in both the error bound and the amount of communication per user.
The precise results obtained depend on the notion of ``neighboring dataset'' in the definition of differential privacy.
We consider the standard notion of neighboring dataset in differential privacy, that the input of a \emph{single user} is changed, and show:
\begin{restatable}{theorem}{mainsingledp}\label{theorem:main_classical_DP}
Let $\epsilon > 0$ and $\delta \in (0,1)$ be any real numbers. There exists a protocol in the shuffled model that is $(\varepsilon,\delta)$-differentially private under single-user changes, has
expected error $O(\frac{1}{\varepsilon} \sqrt{\log\frac{1}{\delta}})$, and where each encoder
sends $O(\log(\frac{n}{\epsilon \delta}))$ messages
of $O(\log(\frac{n}{\delta}))$ bits.
\end{restatable}

We also consider a different notion similar to the gold standard of secure multi-party computation:
Two datasets are considered neighboring if the their \emph{sums} (taken after discretization) are identical.
This notion turns out to allow much better privacy, even with \emph{zero} noise in the final sum --- the only error in the protocol comes from representing the terms of the sum in bounded precision.
\begin{restatable}{theorem}{mainsumdp}\label{theorem:main_sum_DP}
Let $\epsilon > 0$ and $\delta \in (0,1)$ be any real numbers and let $m > 10 \log{(\frac{n}{\epsilon \delta})}$. There exists a protocol in the shuffled model that is $(\varepsilon,\delta)$-differentially private under sum-preserving changes, has
worst-case error $2^{-m}$, and where each encoder sends $m$ messages of $O(m)$ bits.
\end{restatable}
In addition to analyzing error and privacy of our new protocol we consider its resilience towards untrusted users that may deviate from the protocol.
While the shuffled model is vulnerable to such attacks in general~\cite{balle19}, we argue in Section~\ref{sec:resilience} that the privacy guarantees of our protocol are robust even to a large fraction of colluding users.
For reasons of exposition we show Theorem~\ref{theorem:main_sum_DP} before Theorem~\ref{theorem:main_classical_DP}.
The technical ideas behind our new results are discussed in Section~\ref{sec:protocol}.
Next, we discuss implications for machine learning and the relation to previous work.

Concurrently and independently of our work, Balle et al. obtained a result similar to Theorem~\ref{theorem:main_classical_DP} \cite{2019arXiv190609116B}. Their algorithm is similar to ours and, as they point out, a similar algorithm was used by Ishai et al. in their work on cryptography from anonymity \cite{ishai2006cryptography}. Our privacy analysis however is different from theirs. In particular, they use a different noise distribution which leads to a better dependence on $\epsilon$ and $\delta$, achieving a number of messages of $O(\log(n/\delta))$, a message size of $O(\log{n})$ bits and an expected error of $O(1/\epsilon)$.

We point out that following the appearance of this work, tighter quantitative bounds for the secure aggregation problem that was studied in \cite{ishai2006cryptography} were independently obtained in \cite{ghazi2019private_sum_imp, balle_privacy_2019constantIKOS}. Using a reduction of \cite{2019arXiv190609116B}, these imply more efficient differential privacy protocols for aggregation in the shuffled model. Moreover, \cite{ghazi2019private_sum_imp} obtained near tight lower bounds on the corresponding secure aggregation problem. We refer to \cite{ghazi2019private_sum_imp} for a comparison of the various followup works.

\subsection{Discussion of Related Work and Applications}\label{sec:related}

Our protocol is applicable in any setting where secure aggregation is applied.
Below we mention some of the most significant examples and compare to existing results in the literature.

\paragraph{Federated Learning.}
Our main application in a machine learning context is gradient descent-based federated learning~\cite{mcmahan2016communication}.
The idea is to avoid collecting user data, and instead compute weight updates in a distributed manner by sending model parameters to users, locally running stochastic gradient descent on private data, and aggregating model updates over all users. 
Using a \emph{secure aggregation} protocol (see e.g.~\cite{practicalSecAgg}) guards against information leakage from the update of a single user, since the server only learns the aggregated model update.
A federated learning system based on these principles is currently used by Google to train neural networks on data residing on users' phones~\cite{GoogleBlog17}.

Current practical secure aggregation protocols such as that of Bonawitz et al.~\cite{practicalSecAgg} have user computation cost $O(n^2)$ and total communication complexity $O(n^2)$, where $n$ is the number of users.
This limits the number of users that can participate in the secure aggregation protocol.
In addition, the privacy analysis assumes of an ``honest but curious'' server that does not deviate from the protocol, so some level of trust in the secure aggregation server is required.
In contrast, protocols based on shuffling work with much weaker assumptions on the server~\cite{bittau17,cheu19}.
In addition to this advantage, total work and communication of our new protocol scales near-linearly with the number of users.

\paragraph{Differentially Private Aggregation in the Shuffled Model.}
It is known that gradient descent can work well even if data is accessible only in noised form, in order to achieve differential privacy~\cite{abadi2016deep}.
Note that in order to run gradient descent in a differentially private manner, privacy parameters need to be chosen in such a way that the combined privacy loss over many iterations is limited.

\begin{figure}[t]
    \centering
    \def\arraystretch{1.4}
    \begin{tabular}{|l|c|c|c|c|}
        \hline
        {\bf Reference} & 
        {\bf \#messages / $n$} & 
        \begin{tabular}{@{}c@{}}{\bf Message}\\{\bf size}\end{tabular} & 
        {\bf Expected error} & 
        {\bf Privacy protection}\\
        \hline
        \hline
        Cheu et al.~\cite{cheu19} & 
        \begin{tabular}{@{}c@{}} $\varepsilon\sqrt{n}$\\ $m$\end{tabular} &
        1 & 
        \begin{tabular}{@{}c@{}} $\frac{1}{\varepsilon} \log\frac{n}{\delta}$\\ $\sqrt{n} / m + \frac{1}{\varepsilon} \log\frac{1}{\delta}$ \end{tabular} &
        Single-user change\\
        \hline
        Balle et al.~\cite{balle19} & $1$ & $\log n$ & $\frac{n^{1/6}\log^{1/3}(1/\delta)}{\varepsilon^{2/3}}$ & Single-user change\\
        \hline
        \emph{New} &
        \begin{tabular}{@{}c@{}}$\log(\tfrac{n}{\varepsilon \delta})$\\ $m > \log(\tfrac{n}{\varepsilon \delta})$ \end{tabular} &
        \begin{tabular}{@{}c@{}}$\log(\tfrac{n}{\delta})$\\ $m$ \end{tabular} &
        \begin{tabular}{@{}c@{}}$\frac{1}{\varepsilon} \sqrt{\log\frac{1}{\delta}}$\\ $2^{-m}$ \end{tabular} &
        \begin{tabular}{@{}c@{}}Single-user change\\ Sum-preserving change\end{tabular}\\
        \hline
    \end{tabular}
    \caption{Comparison of differentially private aggregation protocols  in the shuffled model with $(\varepsilon,\delta)$-differential privacy. 
    The number of users is $n$, and $m$ is an integer parameter.
    Message sizes are in bits; asymptotic notation is suppressed for readability.
    We consider two types of privacy protection, corresponding to different notions of ``neighboring dataset'' in differential privacy:
    In the first one, which was considered in previous papers, datasets are considered neighboring if they differ in the data of a single user.
    In the latter, datasets are considered neighboring if they have the same sum.
    }
    \label{fig:comparison}
\end{figure}

Each aggregation protocol shown in Figure~\ref{fig:comparison} represents a different trade-off, optimizing different parameters. Our protocols are the only ones that avoid $n^{\Omega(1)}$ factors in both the communication per user and the error.

\paragraph{Private Sketching and Statistical Learning.}
At first glance it may seem that aggregation is a rather weak primitive for combining data from many sources in order to analyze it.
However, research in the area of data stream algorithms has uncovered many non-trivial algorithms that are small \emph{linear sketches}, see e.g.~\cite{cormode2011synopses,woodruff2014sketching}.
Linear sketches over the integers (or over a finite field) can be implemented using secure aggregation by computing linear sketches locally and summing them up over some range that is large enough to hold the sum.
This unlocks many differentially private protocols in the shuffled model, e.g.~estimation of $\ell_p$-norms, quantiles, heavy hitters, and number of distinct elements.

Second, as observed in~\cite{cheu19} we can translate any \emph{statistical query} over a distributed data set to an aggregation problem over numbers in $[0,1]$.
That is, every learning problem solvable using a small number of statistical queries~\cite{kearns1998efficient} can be solved privately and efficiently in the shuffled model.

\subsection{Invisibility Cloak Protocol}\label{sec:protocol}

We use a technique from protocols for secure multi-party aggregation (see e.g.~\cite{secAggSurvey}):  
Ensure that individual numbers passed to the analyzer are fully random by adding random noise terms, but \emph{coordinate} the noise such that all noise terms cancel, and the sum remain the same as the sum of the original data.
Our new insight is that in the shuffled model the addition of zero-sum noise can be done without coordination between the users.
Instead, each user \emph{individually} produces numbers $y_1,\dots,y_m$ that are are fully random except that they sum to $x_i$, and pass them to the shuffler.
This is visualized in Figure~\ref{fig:cloak}.
Conceptually the noise we introduce acts as an \emph{invisibility cloak}: The data is still there, possible to aggregate, but is almost impossible to gain any other information from.

\begin{figure}
    \centering
    \includegraphics[width=0.8\textwidth]{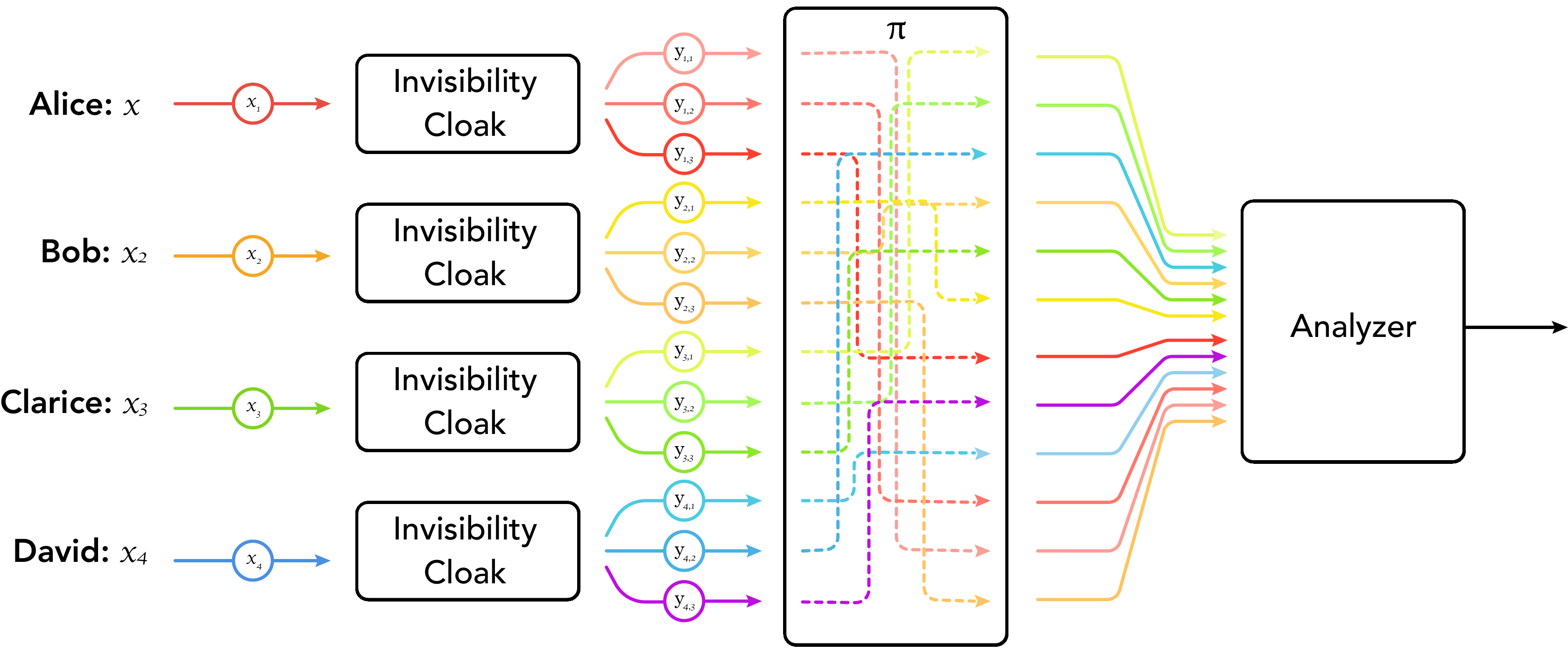}
    \caption{Diagram of the Invisibility Cloak Protocol for secure multi-party aggregation}
    \label{fig:cloak}
\end{figure}

The details of our \emph{encoder} is given as Algorithm~\ref{alg:encoder}.
For parameters $N$, $k$, and $m$ to be specified later it converts each input $x_i$ to a set of random values $\{y_1,\dots,y_m\}$ whose sum, up to scaling and rounding, equals $x_i$.
When the output of all encoders $E_{N,k,m}(x_i)$ is composed with a shuffler this directly gives differential privacy with respect to sum-preserving changes of data (where the sum is considered after rounding).
To achieve differential privacy with respect to single-user changes the protocol must be combined with a pre-randomizer that adds noise to each $x_i$ with some probability, see discussion in Section~\ref{sec:single-user}.

Our \emph{analyzer} is given as Algorithm~\ref{alg:analyzer}.
It computes $\bar{z}$ as the sum of the inputs (received from the shuffler) modulo $N$, which by definition of the encoder is guaranteed to equal the sum $\sum_i \lfloor x_i k\rfloor$ of scaled, rounded inputs.
If $x_1,\dots,x_n \in [0,1]$ this sum will be in $[0,nk]$ and $\bar{z}/k$ will be within $n/k$ of the true sum $\sum_i x_i$.
In the setting where a pre-randomizer adds noise to some inputs, however, we may have $z\not\in [0,nk]$ in which case we round to the nearest feasible output sum, $0$ or $n$.

\paragraph{Privacy Intuition.}
The output of each encoder is very close to fully random in the sense that every set of $m-1$ values are independent and uniformly random.
Only by summing exactly the outputs of an encoder (or several encoders) do we get a value that is not uniformly random.
On the other hand, many size-$m$ subsets \emph{look like} the output of an encoder in the sense that the sum of elements corresponds to a feasible value $x_i$.
In fact, something stronger is true:
For every possible input with the same sum as the true input (sum taken after scaling and rounding) we can, with high probability, find a splitting of the shuffler's output consistent with that input.
Furthermore, the \emph{number} of such splittings is about the same for each potential input.

Our technique can be compared to the recently proposed ``privacy blanket''~\cite{balle19}, which introduces uniform, random noise to replace some inputs.
Since that paper operates in a single-message model there is no possibility of ensuring perfect noise cancellation, and thus the number of noise terms needs to be kept small, which in turn means that a rather coarse discretization is required for differential privacy.
Since the noise we add is zero-sum we can add much more noise, and thus we do not need a coarse discretization, ultimately resulting in much higher accuracy.

\begin{algorithm}[b]

\Fn{$E_{N,k,m}(x)$}{
\KwIn{$x \in \R$, integer parameters $N, k, m \geq 4$}
\KwOut{Multiset $\{y_1,\dots, y_m\} \subseteq \{0,\dots,N-1\}$} \vspace{5pt}

Let $\bar{x} \leftarrow \lfloor x k \rfloor $\\ 
\For{$j = 1,\dots,m-1$}{
	${y}_j \leftarrow \mathit{Uniform}(\{0,\dots,N-1\})$
}
${y}_m \leftarrow \left(\bar{x} - \sum_{j=1}^{m-1} {y}_j\right) \mod N$\\
\Return{$\{{y}_1,\dots, {y}_m\}$}
}
\caption{Invisibility Cloak Encoder Algorithm} \label{alg:encoder}
\end{algorithm}

\begin{algorithm}[t]

\Fn{$A_{N,k,n}(y_{1},\dots, y_{mn})$}{
\KwIn{$(y_{1},\dots, y_{nm})\in \{0,\dots,N-1\}^{mn}$, integer parameters $k$, $n$, odd $N > 3nk$}
\KwOut{$z\in [0,n]$}
\vspace{5pt}

$\bar{z} \leftarrow \sum_i y_{i} \mod N$\\
\lIf{$\bar{z} > 2nk$}{\Return{$0$}}
\lElseIf{$\bar{z} > nk$}{\Return{$n$}}
\lElse{\Return{$\bar{z} / k$}}
}
\caption{Analyzer} \label{alg:analyzer}
\end{algorithm}

\section{Analysis}

\paragraph{Overview.}
We first consider privacy with respect to sum-preserving changes to the input, arguing that observing the output of the shuffler gives almost no information on the input, apart from the sum.
Our proof strategy is to show privacy in the setting of \emph{two} players and then argue that this implies privacy for $n$ players, essentially because the two-player privacy holds regardless of the behavior of the other players.
In the two-player case we first argue that with high probability the outputs of the encoders satisfy a \emph{smoothness} condition saying that every potential input $x_1$, $x_2$ to the encoders corresponds to roughly the same number of divisions of the $2m$ shuffler outputs into sets of size $m$.
Finally we argue that smoothness in conjunction with the $2m$ elements being unique implies privacy.

\subsection{Preliminaries}\label{sec:preliminaries}


\paragraph{Notation.}
We use $\mathit{Uniform}(R)$ to denote a value uniformly sampled from a finite set $R$, and denote by $S_{t}$ the set of all permutations of $\{0,\dots,t-1\}$.
Unless stated otherwise, sets in this paper will be \emph{multisets}.
It will be convenient to work with \emph{indexed} multisets whose elements are identified by indices in some set $I$.
We can represent a multiset $M \subseteq R$ with index set $I$ as a function $M: I\rightarrow R$.
Multisets $M_1$ and $M_2$ with index sets $I_1$ and $I_2$ are considered identical if there exists a bijection $\pi: I_1 \rightarrow I_2$ such that $M_1(i) = M_2(\pi(i))$ for all $i\in I_1$.
For disjoint $I_1$ and $I_2$ we define the union of $M_1$ and $M_2$ as the function defined on $I_1\cup I_2$ that maps $i_1\in I_1$ to $M_1(i_1)$ and $i_2\in I_2$ to $M_2(i_2)$.

\paragraph{Differential Privacy and the Shuffled Model.}
We consider the established notion of differential privacy, formalizing that the output distribution does not differ much between a certain data set and any ``neighboring'' dataset.
\begin{definition}
  \label{def:classical_DP}
  Let $\mathcal{A}$ be a randomized algorithm taking as input a dataset and let $\epsilon \geq 0$ and $\delta \in (0,1)$ be given parameters. Then, $\mathcal{A}$ is said to be \emph{$(\epsilon, \delta)$-differentially private} if for all neighboring datasets $D_1$ and $D_2$ and for all subsets $S$ of the image of $\mathcal{A}$, it is the case that $\Pr[\mathcal{A}(D_1) \in S] \le e^{\epsilon} \cdot \Pr[\mathcal{A}(D_2) \in S] + \delta$, where the probability is over the randomness used by the algorithm $\mathcal{A}$.
\end{definition}
We consider two notions of ``neighboring dataset'': 1) That the input of a single user is changed, but all other inputs are the same%
, and 2) That the sum of user inputs is preserved.
In the latter case we consider the sum after rounding to the nearest lower multiple of $1/k$, for a large integer parameter $k$, i.e., $(x_1,\dots,x_n)\in [0,1]^n$ is a neighbor of $(x'_1,\dots,x'_n)\in [0,1]^n$ if and only if $\sum_i \lfloor x_i k\rfloor = \sum_i \lfloor x'_i k\rfloor$.
(Alternatively, just assume that the input is discretized such that $x_i k$ is integer.)

In the shuffled model, the algorithm that we want to show differentially private is the composition of the shuffler and the encoder algorithm run on user inputs. 
In contrast to the \emph{local model} of differential privacy, the outputs of encoders do not need to be differentially private.
We refer to~\cite{cheu19} for details.

\splitter{
\subsection{Common lemmas}
}
{}

Let $\mathcal{Y} = \{ 0,\dots,N-1 \}$, and consider some indexed multiset $E = \{y_1,\dots,y_{2m}\} \subseteq \mathcal{Y}$ that can possibly be obtained as the union of the outputs of two encoders.
Further, let $\mathcal{I}$ denote the collection of subsets of $\{1,\dots,2m\}$ of size $m$.
For each $I\in\mathcal{I}$ define $X_I(E) = \sum_{i\in I} y_i \mod N$.
We will be interested in the following property of a given (fixed) multiset $E$:
\begin{definition}\label{def:gamma_smooth}
  A multiset $E = \{y_1,\dots,y_{2m}\}$ is \emph{$\gamma$-smooth} if the distribution of values $X_I(E)$ for $I\in\mathcal{I}$ is close to uniform in the sense that $\Pr_{i\in\mathcal{I}}[X_I(E)=x] \in \left[\frac{1-\gamma}{N},\frac{1+\gamma}{N} \right]$ for every $x\in\mathcal{Y}$.
\end{definition}

We name the collection of multisets that are $\gamma$-smooth and contain $2m$ distinct elements:
\[
  \binom{\mathcal{Y}}{2m}_{\gamma\textrm{-smooth}} = \left\{ \{y_1,\dots,y_{2m}\} \; | \;   \{y_1,\dots,y_{2m}\} \textrm{ is $\gamma$-smooth} \textrm{ and } y_1,\dots,y_{2m} \textrm{ are distinct} \right\} \enspace .
\]

Given $x_1, x_2 \in [0,1]$ such that $x_1 k$ and $x_2 k$ are integers, consider the multisets $E_{N,k,m}(x_1) = \{y_1,\dots,y_m\}$ and $E_{N,k,m}(x_2) = \{y_{m+1},\dots,y_{2m}\}$, and let $E(x_1,x_2) = \{y_1,\dots,y_{2m}\}$ be their multiset union.
The multiset $E(x_1,x_2)$ is a random variable due to the random choices made by the encoder algorithm.

\medskip

\begin{lemma}\label{lemma:gamma_smooth}
For every $m\geq 4$, $\gamma > 6\sqrt{m}/2^ {2m}$ and for every choice of $x_1, x_2 \in \mathcal{Y}$ we have $\Pr\left[E(x_1,x_2) \not\in \binom{\mathcal{Y}}{2m}_{\gamma\textrm{-smooth}}\right]  < \frac{2m^2}{N} + \frac{18\sqrt{m}\,N^2}{\gamma^2 2^{2m}} \enspace$.
\end{lemma}
\splitter{
\begin{proof}[Proof of Lemma~\ref{lemma:gamma_smooth}]
    We first upper bound the probability that the multiset $E(x_1,x_2)$ has any duplicate elements.
    For $i\ne j$ consider the event $\mathcal{E}_{i,j}$ that $y_i = y_j$.
    Since $m>2$ we have that every pair of distinct values $y_i, y_j$ are uniform in $\mathcal{Y}$ and independent, so $\Pr(\mathcal{E}_{i,j}) = 1/N$.
    A union bound over all $\binom{2m}{2} < 2m^2$ pairs yields an upper bound of $2m^2/N$ on the probability that there is at least one duplicate pair.

    \medskip
    
    Second, we bound the probability that $E(x_1,x_2)$ is not $\gamma$-smooth.
    Let $I_1 = \{1,\dots,m\}$ and $I_2 = \{m+1,\dots,2m\}$.
    Then by definition of the encoder, $X_{I_1}(E(x_1,x_2)) = x_1$ and $X_{I_2}(E(x_1,x_2)) = x_2$ with probability~1.
    For each $I\in \mathcal{I}\backslash\{I_1,I_2\}$ we have that $X_I$ is uniformly random in the range $\mathcal{Y}$, over the randomness of the encoder.
    Furthermore, observe that the random variables $\{ X_I(E(x_1,x_2)) \}_{I\in\mathcal{I}}$ are pairwise independent.
    Let $Z_I(x)$ be the indicator random variable that is 1 if and only if $X_I(E(x_1,x_2)) = x$.
    Let $\mathcal{I'} = \mathcal{I}\backslash\{I_1,I_2\}$.
    For each $x\in\mathcal{Y}$ and $I\in \mathcal{I'}$ we have $\E[Z_I(x)] = 1/|\mathcal{Y}| = 1/N$.
    The sum $Z(x) = \sum_{I\in\mathcal{I}} Z_I(x)$ equals the number of sets in $\mathcal{I}$ such that $X_I(E(x_1,x_2)) = x$.
    Since $Z_{I_1}(x) = \indicator_{x_1 = x}$ and $Z_{I_2}(x) = \indicator_{x_2 = x}$ it will be helpful to disregard these fixed terms in $Z(x)$.
    Thus we define $Z'(x) = \sum_{I\in\mathcal{I'}} Z_I(x)$, which is a sum of $|\mathcal{I}|-2$ pairwise independent terms, each with expectation $\E[Z_I(x)] = 1/N$.
    Define $\mu = \E[Z'(x)] = |\mathcal{I'}| / N$.
    We bound the variance of $Z'(x)$:
    \[ \Var(Z'(x)) = \E\left[\left(\sum_{I\in\mathcal{I'}} (Z_I(x) - \tfrac{1}{N}) \right)^2\right] = \E\left[\sum_{I\in\mathcal{I'}} \left(Z_I(x) - \tfrac{1}{N}\right)^2 \right] < \E\left[\sum_{I\in\mathcal{I'}} Z_I(x) \right] = \mu \enspace . \]
    The second equality uses that $\E[(Z_{I_1}(x) - \tfrac{1}{N})(Z_{I_2}(x) - \tfrac{1}{N})] = 0$ for $I_1\ne I_2$ because it is a product of two independent, zero-mean random variables.
    The inequality holds because $Z_I(x)$ is an indicator function.
    By Chebychev's inequality over the random choices in the encoder, for any $\sigma > 0$:
    \begin{equation}\label{eq:prob_small_deviation}
    \Pr\left[ | Z'(x) - \mu | > \sigma\mu \right] < \frac{\Var(Z'(x))}{(\sigma\mu)^2} < \frac{1}{\sigma^2 \mu} \enspace .
    \end{equation}
    For $m\geq 4$ we can bound $|\mathcal{I}|-2 = \binom{2m}{m} - 2$ as follows:
    \[ 2^{2m-1} / \sqrt{m} < \binom{2m}{m} - 2 < 2^{2m} / \sqrt{m} \] 
    Using this for upper and lower bounding $\mu$ in (\ref{eq:prob_small_deviation}), and choosing $\sigma = \gamma / 3$ we get:
    \begin{equation*}
    \Pr\left[ | Z'(x) - \mu | > \gamma\, 2^{2m} / (3N\sqrt{m}) \right] < \frac{18 \sqrt{m} N}{\mu^2 2^m} \enspace .
    \end{equation*}
    A union bound over all $x\in\mathcal{Y}$ implies that with probability at least $1 - \frac{18\sqrt{m}\,N^2}{\gamma^2 2^{2m}}$:
    \begin{equation}\label{success_prob}
    \forall x\in\mathcal{Y} \colon | Z'(x) - \mu | \leq \gamma\, 2^{2m} / (3N\sqrt{m})
    \end{equation}
    Conditioned on (\ref{success_prob}) we have:
    \begin{align*}
        \Pr_{i\in\mathcal{I}}[X_I(E(x_1,x_2))=x]  = Z(x)/|\mathcal{I}|
        & \leq (Z'(x)+2) / |\mathcal{I}|\\
        & \leq \frac{\mu + 2 + \gamma\, 2^{2m} / (2 N\sqrt{m})}{|\mathcal{I}|}\\
        & \leq \frac{1}{N} + \frac{2+ \gamma\, 2^{2m} / (3N\sqrt{m})}{2^{2m-1} / \sqrt{m}}\\
        & = \frac{1+\tfrac{\sqrt{m}}{2^{2m-1}}+2\gamma/3}{N}
         \leq \frac{1+\gamma}{N} \enspace .
    \end{align*}
    The final inequality uses the assumption that $\gamma > 6\sqrt{m}/2^ {2m}$.
    A similar computation shows that conditioned on (\ref{success_prob}), $\Pr_{i\in\mathcal{I}}[X_I(E(x_1,x_2))=x] \geq \frac{1-\gamma}{N}$.
\end{proof}
}
{The proof of Lemma~\ref{lemma:gamma_smooth} appears in the supplementary material.}

\splitter{
\begin{corollary}\label{corollary:gamma_smooth}
For $m\geq 4$, and $m = 3 \lceil \log N\rceil$,
\[\Pr\left[E(x_1,x_2) \not\in \binom{\mathcal{Y}}{2m}_{N^{-1}\textrm{-smooth}}\right]  < \frac{19 \lceil\log N\rceil^2}{N} \enspace .\]
\end{corollary}
\begin{proof}
    We invoke Lemma~\ref{lemma:gamma_smooth} with $\gamma = N^{-1}$ and $m = 3 \lceil\log N\rceil$.
    The probability bound is
    \[\frac{18 \lceil \log N\rceil^2}{N} + \frac{18\sqrt{3\lceil \log N\rceil}\, N^2}{ N^{-2} \, 2^{6\lceil \log N\rceil}}
    < \frac{18 \lceil \log N\rceil^2}{N} + \frac{18 \lceil \log N\rceil}{ N^{2}} \enspace .\]
    Because $\log N\geq 3$ and $N \geq 6$ this shows the stated bound.
\end{proof}
}
{}

 Denote by $E(x_1,x_2;y_1,\dots,y_{m-1}, y_{m+1}, \dots, y_{2m-1})$ the sequence obtained by the deterministic encoding for given values $y_1, \dots, y_{m-1}, y_{m+1}, \dots, y_{2m-1} \in \mathcal{Y}$ in Algorithm~\ref{alg:encoder}.
 \splitter{Moreover, we denote by $\overline{E}(x_1, x_2, y_1,\dots,y_{m-1}, y_{m+1}, \dots, y_{2m-1})$ the corresponding multiset.

\begin{lemma}\label{lemma:two_player_set_probability}
    For any $y^* \in \binom{\mathcal{Y}}{2m}$ and for any $x_1$ and $x_2$, it is the case that
    \begin{equation*}
      \Pr[E(x_1, x_2) = y^*] = \frac{1}{|\mathcal{Y}|^{2 (m-1)}} \cdot \displaystyle\sum\limits_{\pi \in S_{2m}} \indicator_{E(x_1,x_2;\pi(y^*)_1,\dots,\pi(y^*)_{m-1},\pi(y^*)_{m+1}, \dots, \pi(y^*)_{2m-1}) = \pi(y^*)}.
    \end{equation*}
\end{lemma}

\begin{proof}[Proof of Lemma~\ref{lemma:two_player_set_probability}]
Using the fact that all the elements in $y^*$ are distinct, we have that
    \begin{align}
        \Pr[E(x_1, x_2) = y^*] &= \displaystyle\sum\limits_{\genfrac{}{}{0pt}{2}{y_1, \dots, y_{m-1},}{y_{m+1}, \dots, y_{2m-1} \in \mathcal{Y}}}
        \frac{1}{|\mathcal{Y}|^{2 (m-1)}} \cdot \indicator_{\overline{E}(x_1,x_2;y_1,\dots,y_{m-1},y_{m+1}, \dots, y_{2m-1}) = y^*}\nonumber\\ 
        &= \frac{1}{|\mathcal{Y}|^{2 (m-1)}} \cdot \displaystyle\sum\limits_{\genfrac{}{}{0pt}{2}{\text{distinct } y_1, \dots, y_{m-1}}{y_{m+1}, \dots, y_{2m-1} \in \mathcal{Y}}}
        \indicator_{\overline{E}(x_1,x_2;y_1,\dots,y_{m-1},y_{m+1}, \dots, y_{2m-1}) = y^*}\nonumber\\ 
        &= \frac{1}{|\mathcal{Y}|^{2 (m-1)}} \cdot \displaystyle\sum\limits_{\genfrac{}{}{0pt}{2}{\text{distinct } y_1, \dots, y_{m-1}}{y_{m+1}, \dots, y_{2m-1} \in \mathcal{Y}}} \displaystyle\sum\limits_{\pi \in S_{2m}} \indicator_{E(x_1,x_2;y_1,\dots,y_{m-1},y_{m+1}, \dots, y_{2m-1}) = \pi(y^*)}\nonumber\\ 
        &= \frac{1}{|\mathcal{Y}|^{2 (m-1)}} \cdot \displaystyle\sum\limits_{\pi \in S_{2m}}
        \displaystyle\sum\limits_{\genfrac{}{}{0pt}{2}{\text{distinct } y_1, \dots, y_{m-1}}{y_{m+1}, \dots, y_{2m-1} \in \mathcal{Y}}} \indicator_{E(x_1,x_2;y_1,\dots,y_{m-1},y_{m+1}, \dots, y_{2m-1}) = \pi(y^*)}\nonumber\\ 
        &= \frac{1}{|\mathcal{Y}|^{2 (m-1)}} \cdot \displaystyle\sum\limits_{\pi \in S_{2m}} \indicator_{E(x_1,x_2 ; \pi(y^*)_1,\dots,\pi(y^*)_{m-1}, \pi(y^*)_{m+1}, \dots, \pi(y^*)_{2m-1}) = \pi(y^*)}\nonumber
    \end{align}
\end{proof}
}
{}

\subsection{Analysis of Privacy under Sum-Preserving Changes}\label{sec:sum-preserving}

\begin{lemma}\label{lemma:pair_pointwise_DP}
    For any $y^* \in \binom{\mathcal{Y}}{2m}_{\gamma\textrm{-smooth}}$ and for all $x_1, x_2, x'_1, x'_2$ that are integer multiples of $1/k$ and that satisfy $x_1 + x_2 = x'_1 +x'_2$, it is the case that $\Pr[E(x_1, x_2) = y^*] \le \frac{1+\gamma}{1-\gamma} \cdot \Pr[E(x'_1, x'_2) = y^*]$.
\end{lemma}

\splitter{
\begin{proof}[Proof of Lemma~\ref{lemma:pair_pointwise_DP}]
     We denote by $\sum_i y^*_i := \sum_{i \in [2m]} y^*_i$ the sum of all elements in the set $y^*$. We define
\begin{equation}\label{eq:number_x_1_def_SPDP}
    B_{y^*, x_1} := \text{Number of subsets } S \text{ of } \{1, \dots, 2m\} \text{ of size }m \text{ for which } \sum_{i\in S} y^*_i = x_1 k \mod N.
\end{equation}
We similarly define $B_{y^*, x'_1}$ by replacing $x_1$ in~(\ref{eq:number_x_1_def_SPDP}) by $x'_1$.\\ 
Since $y^* \in \binom{\mathcal{Y}}{2m}$, Lemma~\ref{lemma:two_player_set_probability} implies that
    \begin{align}
         \Pr[E(x_1, x_2) = y^*] &= \frac{1}{|\mathcal{Y}|^{2 (m-1)}} \cdot \displaystyle\sum\limits_{\pi \in S_{2m}} \indicator_{E(x_1,x_2;\pi(y^*)_1,\dots,\pi(y^*)_{m-1}, \pi(y^*)_{m+1}, \dots, \pi(y^*)_{2m-1}) = \pi(y^*)}\nonumber\\ 
         &= \frac{(m!)^2}{|\mathcal{Y}|^{2 (m-1)}} \cdot B_{y^*, x_1} \cdot \indicator_{\sum_i y^*_i = x_1 k + x_2 k} .\label{eq:x_prob}
    \end{align}
    Similarly, we have that
    \begin{equation}\label{eq:x_primes_prob}
        \Pr[E(x'_1, x'_2) = y^*] = \frac{(m!)^2}{|\mathcal{Y}|^{2 (m-1)}} \cdot B_{y^*, x'_1} \cdot \indicator_{\sum_i y^*_i = x'_1 k + x'_2 k}.
    \end{equation}
    Since $y^*$ is $\gamma$-smooth, Definition~\ref{def:gamma_smooth} implies that \begin{equation}\label{eq:ratio_numbers_upper_bound_SPDP}
        \frac{B_{y^*, x_1}}{B_{y^*, x'_1}} \le \frac{1+\gamma}{1-\gamma}.
    \end{equation}
    By Equations~(\ref{eq:x_prob}) and~(\ref{eq:x_primes_prob}) and the assumption that $x_1+x_2 = x'_1+x'_2$ (as well as the assumption that $x_1, x_2, x'_1, x'_2$ are all integer multiples of $1/k$), we get that for every $\gamma$-smooth $y^*$ whose sum is \emph{not equal} to $x_1 k + x_2 k$, it is the case that
    \begin{equation}
        \Pr[E(x_1, x_2) = y^*] = \Pr[E(x'_1, x'_2) = y^*] = 0,
    \end{equation}
    and for every $\gamma$-smooth $y^*$ whose sum is \emph{equal} to $x_1 k + x_2 k$, the ratio of Equations~(\ref{eq:x_prob}) and~(\ref{eq:x_primes_prob}) along with~(\ref{eq:ratio_numbers_upper_bound_SPDP}) give that
    \begin{equation}
        \Pr[E(x_1, x_2) = y^*] \le \frac{1+\gamma}{1-\gamma} \cdot \Pr[E(x'_1, x'_2) = y^*].
    \end{equation}
\end{proof}
}
{The proof of Lemma~\ref{lemma:pair_pointwise_DP} is given in the supplementary material.}

\begin{lemma}\label{lemma:swap_dp}
 Suppose $x_1, x_2, \dots, x_n \in \R$ and $x_1', x_2', \dots, x_n' \in \R$ that are integer multiples of $1/k$ satisfying $x_i = x_i'$ for all $i\neq j_1, j_2$, where $1\leq j_1\neq j_2 \leq n$. Moreover, suppose that for any set $T$ consisting of multisets of $2m$ elements from $\mathcal{Y}$, we have \splitter{the following guarantee:
 \begin{equation}
  \Pr[E(x_{j_1}, x_{j_2}) \in T] \leq e^\epsilon \cdot \Pr[E(x_{j_1}', x_{j_2}') \in T] + \delta \label{eq:swapcond}
 \end{equation}
 }
 {that $\Pr[E(x_{j_1}, x_{j_2}) \in T] \leq e^\epsilon \cdot \Pr[E(x_{j_1}', x_{j_2}') \in T] + \delta$}
 for some $\epsilon, \delta > 0$. Then, it follows that for any set $S$ of multisets consisting of $mn$ elements from $\mathcal{Y}$,
 \[
  \Pr[E(x_1, x_2, \dots, x_n) \in S] \leq e^\epsilon \cdot \Pr[E(x_1', x_2', \dots, x_n') \in S] + \delta.
 \]
\end{lemma}
\splitter{
\begin{proof}[Proof of Lemma~\ref{lemma:swap_dp}]
  Without loss of generality, assume $j_1=1$ and $j_2=2$ (by symmetry). Thus, $x_i = x_i'$ for $i=3,\dots,n$. For ease of notation, let $x = (x_1, x_2, \dots, x_n)$ and $x' = (x_1', x_2', \dots, x_n')$.
  
  Suppose $S$ is an arbitrary set of multisets of $mn$ elements from $\mathcal{Y}$. For any $A \subset \mathcal{Y}^m$, we let $\mathcal{R}_{S, A}$ denote
  \[
    \mathcal{R}_{S, A} = \bigcup_{T\in S} \left(T \setminus \bigcup_{a\in A} \left\{ a_1, a_2, \dots, a_m \right\}\right).
  \]
  
  Then, we observe that
  \begin{align}
      \Pr[E(x) \in S] &= \sum_{y_3, \dots, y_n \in \mathcal{Y}^m} \Pr[E(x) \in S \mid \forall i>2, E(x_i) = y_i] \cdot \prod_{j=3}^n \Pr[E(x_j) = y_j] \nonumber\\
      &= \sum_{y_3,\dots,y_n} \Pr\left[E(x_1, x_2) \in \mathcal{R}_{S, \{y_3, y_4, \dots, y_n\}} \right] \cdot \prod_{j=3}^n \Pr[E(x_j) = y_j] \nonumber\\
      &= \sum_{y_3,\dots, y_n} \left(e^\epsilon \cdot \Pr\left[E(x_1',x_2')\in \mathcal{R}_{S, \{y_3,y_4,\dots, y_n\}}\right] + \delta \right) \cdot \prod_{j=3}^n \Pr[E(x_j') = y_j] \label{eq:swapstep}\\
      &\leq e^\epsilon \cdot \Pr\left[E(x') \in S\right] + \delta \cdot \sum_{y_3,\dots,y_n \in \mathcal{Y}} \left( \prod_{j=3}^n \Pr[E(x_j') = y_j]\right) \nonumber\\
      &\leq e^\epsilon \cdot \Pr\left[E(x') \in S\right] + \delta, \nonumber
  \end{align}
  where \eqref{eq:swapstep} follows from \eqref{eq:swapcond} and the fact that $x_i = x_i'$ for $i=3,4,\dots,n$. This completes the proof.
\end{proof}
}
{The proof of Lemma~\ref{lemma:swap_dp} is also given in the supplementary material.}

\begin{lemma}\label{lemma:sumpreservingswap_dp}
 Suppose $x_1, x_2, \dots, x_n \in \R$ and $x_1', x_2', \dots, x_n' \in \R$ such that $x_{j_1} + x_{j_2} = x_{j_1}' + x_{j_2}'$ (each of these being an integer multiple of $1/k$) and $x_i = x_i'$ for all $i\neq j_1, j_2$, where $1\leq j_1\neq j_2 \leq n$. Then, for any set $S$ of multisets consisting of $mn$ elements from $\mathcal{Y}$, we have $\Pr[E(x_1, x_2, \dots, x_n) \in S] \leq \frac{1+\gamma}{1-\gamma} \cdot \Pr[E(x_1', x_2', \dots, x_n') \in S] + \eta$, where $\eta = \frac{2m^2}{N} + \frac{18\sqrt{m}N^2}{\gamma^2 2^{2m}}$ and $\gamma > \frac{6\sqrt{m}}{2^{2m}}$.
\end{lemma}
\splitter{
\begin{proof}[Proof of Lemma~\ref{lemma:sumpreservingswap_dp}]
Without loss of generality, let $j_1=1$ and $j_2=2$. We now consider any set $T$ of multisets of $2m$ elements from $\mathcal{Y}$. Observe that
  \begin{align}
      \Pr[E(x_1, x_2) \in T] &\leq \Pr\left[E(x_1,x_2) \not\in \smoothset\right] + \Pr\left[E(x_1,x_2) \in T \cap \smoothset\right] \nonumber\\
      &\leq \eta + \sum_{A\in T\cap\smoothset} \Pr[E(x_1,x_2)=A] \label{eq:nonsmooth}\\
      &\leq \eta + \sum_{A\in T\cap\smoothset} \frac{1+\gamma}{1-\gamma} \cdot \Pr[E(x_1', x_2')=A] \label{eq:twodp}\\
      &\leq \eta + \frac{1+\gamma}{1-\gamma} \cdot \Pr[E(x_1', x_2') \in T], \nonumber
  \end{align}
  where \eqref{eq:nonsmooth} and \eqref{eq:twodp} follow from Lemma~\ref{lemma:gamma_smooth} and Lemma~\ref{lemma:pair_pointwise_DP}, respectively. The desired result now follows from a direct application of Lemma~\ref{lemma:swap_dp}.
\end{proof}
}
{The proof of Lemma~\ref{lemma:sumpreservingswap_dp} is given in the supplementary material.}
Using Lemma~\ref{lemma:sumpreservingswap_dp} as a building block for analyzing differential privacy guarantees in the context of sum-preserving \emph{swaps}, we can derive a differential privacy result with respect to general sum-preserving \emph{changes}.
\begin{lemma}\label{lemma:sum_swap_series}
  Suppose $x = (x_1, x_2, \dots, x_n)$ and $x' = (x_1', x_2', \dots, x_n')$ have coordinates that are integer multiples of $1/k$ satisfying $x_1 + x_2 + \cdots + x_n = x_1' + x_2' + \cdots + x_n'$ and $x'$ can be obtained from $x$ by a series of $t$ sum-preserving swaps. Then, for any $S$, we have $\Pr[E(x_1',x_2',\dots,x_n') \in S] \leq \beta^t \Pr[E(x_1,x_2,\dots,x_n) \in S] + \eta \cdot \frac{\beta^t-1}{\beta-1}$, where $\eta = \frac{2m^2}{N} + \frac{18\sqrt{m}N^2}{\gamma^2 2^{2m}}$, $\gamma > \frac{6\sqrt{m}}{2^{2m}}$, and $\beta = \frac{1+\gamma}{1-\gamma}$.
\end{lemma}
\splitter{
\begin{proof}[Proof of Lemma~\ref{lemma:sum_swap_series}]
  We prove the lemma by induction on $t$. Note that the case $t=1$ holds by Lemma~\ref{lemma:sumpreservingswap_dp}.
  
  Now, for the inductive step, suppose the lemma holds for $t = r$. We wish to show that it also holds for $t=r+1$. Note that there exists some $x'' \in \mathcal{Y}^n$ such that (1) $x''$ can be obtained from $x$ by a series of $r$ sum-preserving swaps and (2) $x'$ can be obtained from $x''$ by a single sum-preserving swap. By the inductive hypothesis, we have that
  \begin{equation}
    \Pr[E(x_1'', x_2'', \dots, x_n'') \in S] \leq \beta^r \Pr[E(x_1,x_2,\dots,x_n) \in S] + \eta\cdot\frac{\beta^r-1}{\beta - 1}. \label{eq:indhyp}
  \end{equation}
  Moreover, by Lemma~\ref{lemma:swap_dp}, we have that
  \begin{equation}
    \Pr[E(x_1', x_2', \dots, x_n') \in S] \leq \beta \Pr[E(x_1'',x_2'',\dots,x_n'') \in S] + \eta. \label{eq:indstep}
  \end{equation}
  Combining \eqref{eq:indhyp} and \eqref{eq:indstep}, we note that
  \begin{align*}
    \Pr[E(x_1', x_2', \dots, x_n') \in S] &\leq \beta \left( \beta^r \Pr[E(x_1,x_2,\dots,x_n) \in S] + \eta\cdot\frac{\beta^r-1}{\beta - 1} \right) + \eta \\
    &\leq \beta^{r+1} \Pr[E(x_1,x_2,\dots,x_n) \in S] + \eta\cdot\frac{\beta^{r+1} - 1}{\beta - 1},
  \end{align*}
  which establishes the claim for $t=r+1$.
\end{proof}
}
{The proof of Lemma~\ref{lemma:sum_swap_series} appears in the supplementary material.}
\splitter{
As a consequence, we obtain the following main theorem establishing differential privacy of Algorithm~\ref{alg:encoder} with respect to sum-preserving changes in the shuffled model:

\mainsumdp*
}
{We are now ready to prove Theorem~\ref{theorem:main_sum_DP}.}
\begin{proof}[Proof of Theorem~\ref{theorem:main_sum_DP}]

In Algorithm~\ref{alg:encoder}, each user communicates at most $O(m \log{N})$ bits which are sent via $m$ messages. Note that if $x = (x_1, \dots, x_n)$ and $x' = (x_1', \dots, x_n')$ have coordinates that are integer multiples of $1/k$ satisfying $x_1 + \cdots + x_n = x_1' + \cdots + x_n'$, then there is a sequence of $t \leq n-1$ sum-preserving swaps that allows us to transform $x$ into $x'$. Thus, Lemma~\ref{lemma:sum_swap_series} implies that Algorithm~\ref{alg:encoder} is $(\epsilon, \delta)$-differentially private with respect to sum-preserving changes if $\frac{(1+\gamma)^{n-1}}{(1-\gamma)^{n-1}} \le e^{\epsilon}$, and $\frac{2m^2}{N} + \frac{18\sqrt{m}\,N^2}{\gamma^2 2^{2m}} \le \delta$, for any $\gamma > \frac{6\sqrt{m}}{2^{2m}}$ and $m\geq 4$. The error in our final estimate (which is due to rounding) is $O(n/k)$ in the worst case. The theorem now follows by choosing $m > 10\log\bigg(\frac{nk}{\epsilon \delta}\bigg)$, $\gamma = \frac{\epsilon}{10 n}$, $k = 10n$ and $N$ being the first odd integer larger than $3kn+\frac{10}{\delta} + \frac{10}{\epsilon}$.
\end{proof}

\subsection{Analysis of Privacy under Single-User Changes}\label{sec:single-user}

The main idea is to run Algorithm~\ref{alg:encoder} after having each player add some noise to her input, with some fixed probability independently from the other players. We need the noise distribution to satisfy three properties: it should be supported on a finite interval, the logarithm of its probability mass function should have a small Lipschitz-constant (even under modular arithmetic) and its variance should be small. The following truncated version of the discrete Laplace distribution satisfies all three properties.

\begin{definition}[Truncated Discrete Laplace Distribution]
    Let $N$ be a positive odd integer and $p \in (0,1)$. The probability mass function of the truncated discrete Laplace distribution $\D_{N,p}$ is defined by
  \begin{equation}\label{eq::noise_distribution}
  \D_{N, p}[k] = \frac{(1-p)\cdot p^{|k|}}{1+p-2p^{\frac{N+1}{2}}} \end{equation}
    for every integer $k$ in the range $\{-\frac{(N-1)}{2}, \dots, +\frac{(N-1)}{2}\}$.
\end{definition}
    
\begin{lemma}[Log-Lipschitzness]\label{lemma:noise_distribution_ratio}
 Let $N$ be a positive odd integer and $p \in (0,1)$ a real number. Define the interval $I = \{-\frac{(N-1)}{2}, \dots, +\frac{(N-1)}{2}\}$. For all $k \in \{0,\dots, N-1\}$ and all $t \in I$, it is the case that $p^{|t|} \le \frac{\D_{N, p}[(k+t) \mod I]}{\D_{N, p}[k \mod I]} \le p^{-|t|}$.
\end{lemma}

\splitter{
\begin{proof}[Proof of Lemma~\ref{lemma:noise_distribution_ratio}]

We start by noting that (\ref{eq::noise_distribution}) implies that
    \begin{equation}\label{eq:general_ratio}
        \frac{\D_{N, p}[(k+t) \mod I]}{\D_{N, p}[k \mod I]} = \frac{p^{|(k+t) \mod I|} }{p^{|k \mod I|} }.
    \end{equation}

We distinguish six cases depending on the values of $k$ and $k+t$:
    \begin{align}
        &\text{Case }1:~ 0 \le k \le \frac{N-1}{2} ~  \text{ and } ~ -\frac{(N-1)}{2} \le k+t \le -1.\label{eq:case_1}\\ 
        &\text{Case }2:~ 0 \le k \le \frac{N-1}{2} ~  \text{ and } ~ 0 \le k+t \le \frac{N-1}{2}.\label{eq:case_2}\\ 
        &\text{Case }3:~ 0 \le k \le \frac{N-1}{2} ~  \text{ and } ~ \frac{N+1}{2} \le k+t \le N -1.\label{eq:case_3}\\ 
        &\text{Case }4:~ \frac{N+1}{2} \le k \le N-1 ~  \text{ and } ~ 1 \le k+t \le \frac{N-1}{2}.\label{eq:case_4}\\ 
        &\text{Case }5:~ \frac{N+1}{2} \le k \le N-1 ~  \text{ and } ~ \frac{N+1}{2} \le k+t \le N -1.\label{eq:case_5}\\ 
        &\text{Case }6:~ \frac{N+1}{2} \le k \le N-1 ~  \text{ and } ~ N \le k+t \le N-1+\frac{N-1}{2}.\label{eq:case_6}
    \end{align}
    In Cases $1$, $2$ and $3$, we have that $0 \le k \le \frac{N-1}{2}$ which implies that $|k \mod I| = k$ and hence the denominator in~(\ref{eq:general_ratio}) satisfies
    \begin{equation}\label{eq:denominator_cases_1_2_3}
        p^{|k \mod I|} = p^{k}.
    \end{equation}
    Plugging~(\ref{eq:denominator_cases_1_2_3}) in~(\ref{eq:general_ratio}), we get
    \begin{equation}\label{eq:ratio_cases_1_2_3}
        \frac{\D_{N, p}[(k+t) \mod I]}{\D_{N, p}[k \mod I]} = \frac{p^{|(k+t) \mod I|} }{p^{k} }.
    \end{equation}    
    We now separately examine each of these three cases.
    \paragraph{Case $\pmb{1}$.}
    If~(\ref{eq:case_1}) holds, then $|(k+t) \mod I| = -k-t$ and the numerator in~(\ref{eq:ratio_cases_1_2_3}) becomes
    \begin{equation}\label{eq:numerator_case_1}
        p^{|(k+t) \mod I|}= p^{-k-t}.
    \end{equation}
    Plugging~(\ref{eq:numerator_case_1}) in~(\ref{eq:ratio_cases_1_2_3}), we get
    \begin{equation}\label{eq:simplified_ratio_case_1}
         \frac{\D_{N, p}[(k+t) \mod I]}{\D_{N, p}[k \mod I]} = p^{-2k-t}.
     \end{equation}
     Using the facts that $k+t < 0$ and $k \geq 0$, and thus that $t<0$, we get that the quantity in~(\ref{eq:simplified_ratio_case_1}) is at most $p^{-|t|}$ and at least $p^{|t|}$.
     
    \paragraph{Case $\pmb{2}$.}
    If~(\ref{eq:case_2}) holds, then $|(k+t) \mod I| = k+t$ and the numerator in~(\ref{eq:ratio_cases_1_2_3}) becomes
    \begin{equation}\label{eq:numerator_case_2}
        p^{|(k+t) \mod I|}= p^{k+t}.
    \end{equation}
    Plugging~(\ref{eq:numerator_case_2}) in~(\ref{eq:ratio_cases_1_2_3}), we get
    \begin{equation*}
         \frac{\D_{N, p}[(k+t) \mod I]}{\D_{N, p}[k \mod I]} = p^t.
     \end{equation*}
    
    \paragraph{Case $\pmb{3}$.}
    If~(\ref{eq:case_3}) holds, then $|(k+t) \mod I| = N-k-t$ and the numerator in~(\ref{eq:ratio_cases_1_2_3}) becomes
    \begin{equation}\label{eq:numerator_case_3}
        p^{|(k+t) \mod I|}= p^{N-k-t}.
    \end{equation}
    Plugging~(\ref{eq:numerator_case_3}) in~(\ref{eq:ratio_cases_1_2_3}), we get
    \begin{equation}\label{eq:simplified_ratio_case_3}
         \frac{\D_{N, p}[(k+t) \mod I]}{\D_{N, p}[k \mod I]} = p^{N-2k-t}.
     \end{equation}
     Using the fact that $k+t \geq \frac{N+1}{2}$ which, along with the fact that $k \le \frac{N-1}{2}$, implies that $t>0$, we get that the quantity in~(\ref{eq:simplified_ratio_case_3}) is at most $p^{-|t|}$ and at least $p^{|t|}$.\\ 
     
    We now turn to Cases $4$, $5$ and $6$. In these, $\frac{N+1}{2} \le k \le N-1$, which implies that $|k \mod I| = N-k$ and hence the denominator in~(\ref{eq:general_ratio}) satisfies
    \begin{equation}\label{eq:denominator_cases_4_5_6}
        p^{|k \mod I|} = p^{N-k}.
    \end{equation}
    Plugging~(\ref{eq:denominator_cases_4_5_6}) in~(\ref{eq:general_ratio}), we get
    \begin{equation}\label{eq:ratio_cases_4_5_6}
        \frac{\D_{N, p}[(k+t) \mod I]}{\D_{N, p}[k \mod I]} = \frac{p^{|(k+t) \mod I|} }{p^{N-k} }.
    \end{equation}    
    We now separately examine each of these three cases.
    
    \paragraph{Case $\pmb{4}$.}
    If~(\ref{eq:case_4}) holds, then $|(k+t) \mod I| = k+t$ and the numerator in~(\ref{eq:ratio_cases_4_5_6}) becomes
    \begin{equation}\label{eq:numerator_case_4}
        p^{|(k+t) \mod I|}= p^{k+t}.
    \end{equation}
    Plugging~(\ref{eq:numerator_case_4}) in~(\ref{eq:ratio_cases_4_5_6}), we get
    \begin{equation}\label{eq:simplified_ratio_case_4}
         \frac{\D_{N, p}[(k+t) \mod I]}{\D_{N, p}[k \mod I]} = p^{2k+t-N}.
     \end{equation}
     Using the facts that $k+t \le \frac{N-1}{2}$ and $k \geq \frac{N+1}{2}$, we deduce that $t<0$ and that the quantity in~(\ref{eq:simplified_ratio_case_4}) is at most $p^{-|t|}$ and at least $p^{|t|}$.
     
    \paragraph{Case $\pmb{5}$.}
    If~(\ref{eq:case_5}) holds, then $|(k+t) \mod I| = N-k-t$ and the numerator in~(\ref{eq:ratio_cases_4_5_6}) becomes
    \begin{equation}\label{eq:numerator_case_5}
        p^{|(k+t) \mod I|}= p^{N-k-t}.
    \end{equation}
    Plugging~(\ref{eq:numerator_case_5}) in~(\ref{eq:ratio_cases_4_5_6}), we get
    \begin{equation*}
         \frac{\D_{N, p}[(k+t) \mod I]}{\D_{N, p}[k \mod I]} = p^{-t}.
     \end{equation*}     
     
    \paragraph{Case $\pmb{6}$.}
    If~(\ref{eq:case_6}) holds, then $|(k+t) \mod I| = k+t-N$ and the numerator in~(\ref{eq:ratio_cases_4_5_6}) becomes
    \begin{equation}\label{eq:numerator_case_6}
        p^{|(k+t) \mod I|}= p^{k+t-N}.
    \end{equation}
    Plugging~(\ref{eq:numerator_case_6}) in~(\ref{eq:ratio_cases_4_5_6}), we get
    \begin{equation}\label{eq:simplified_ratio_case_6}
         \frac{\D_{N, p}[(k+t) \mod I]}{\D_{N, p}[k \mod I]} = p^{2k+t-2N}.
     \end{equation}
     Using the facts that $k < N$ and $k+t \geq N$, we get that $t>0$ and that the quantity in~(\ref{eq:simplified_ratio_case_6}) is at most $p^{-|t|}$ and at least $p^{|t|}$.
\end{proof}
}
{The proof of Lemma~\ref{lemma:noise_distribution_ratio} appears in the supplementary material.}

\begin{lemma}\label{lemma:distribution_mean_variance}
 Let $N$ be a positive odd integer and $p \in (0,1)$ a real number. Let $X$ be a random variable drawn from the truncated discrete Laplace distribution $\D_{N,p}$. Then, the mean and variance of $X$ satisfy $\E[X] = 0$ and $\Var[X] \le \frac{2 p (1+p)}{(1-p)^2 (1+p-2p^{(N+1)/2})}$.
\end{lemma}

\splitter{
In order to prove Lemma~\ref{lemma:distribution_mean_variance}, we will need the simple fact given in Lemma~\ref{lemma:second_moment_geometric}.
\begin{lemma}\label{lemma:second_moment_geometric}
For any $p \in [0,1)$, it is the case that $\displaystyle\sum\limits_{k=1}^{\infty} k^2 p^{k} = \frac{p(1+p)}{(1-p)^3}$.
\end{lemma}

\begin{proof}[Proof of Lemma~\ref{lemma:second_moment_geometric}]
    For every $p \in [0,1)$, we consider the geometric series $f(p) := \displaystyle\sum\limits_{k=1}^{\infty} p^{k}$. Differentiating and multiplying by $p$, we get $p f'(p) = \displaystyle\sum\limits_{k=1}^{\infty} k p^{k}$. Differentiating a second time and multiplying by $p$, we get
    \begin{equation}\label{eq:second_moment_geometric}
    p (p f'(p))' = \displaystyle\sum\limits_{k=1}^{\infty} k^2 p^{k}.    
    \end{equation}
    Using the formula for a convergent geometric series, we have $f(p) = \frac{p}{1-p}$. Plugging this expression in (\ref{eq:second_moment_geometric}) and differentiating, we get $\displaystyle\sum\limits_{k=1}^{\infty} k^2 p^{k} = \frac{p(1+p)}{(1-p)^3}$.
\end{proof}

\begin{proof}[Proof of Lemma~\ref{lemma:distribution_mean_variance}]
    We have that
    \begin{equation}\label{eq:noise_mean_calculation}
        \E[X] = \displaystyle\sum\limits_{k=-(N-1)/2}^{(N-1)/2} k \cdot \D_{N, p}[k] = \displaystyle\sum\limits_{k=1}^{(N-1)/2} k \cdot (\D_{N, p}[k] - \D_{N, p}[-k]) = 0,
    \end{equation}
    where the last equality follows from the fact that $\D_{N, p}[k] = \D_{N, p}[-k]$ for all $k \in \{1, \dots, (N-1)/2\}$ (which directly follows from (\ref{eq::noise_distribution})). Using this same property along with~(\ref{eq:noise_mean_calculation}), we also get that
    \begin{equation}\label{eq:noise_variance}
        \Var[X] = \E[X^2] = \displaystyle\sum\limits_{k=-(N-1)/2}^{(N-1)/2} k^2 \cdot \D_{N, p}[k] = 2 \cdot \displaystyle\sum\limits_{k=1}^{(N-1)/2} k^2 \cdot \D_{N, p}[k]
    \end{equation}
    Plugging the definition (\ref{eq::noise_distribution}) of $\D_{N, p}[k]$ in (\ref{eq:noise_variance}), we get
    \begin{align}
      \Var[X] &= \frac{2(1-p)}{(1+p)(1+p-2p^{(N+1)/2})} \displaystyle\sum\limits_{k=1}^{(N-1)/2} k^2 p^{k}\nonumber\\ 
      &\le \frac{2(1-p)}{(1+p)(1+p-2p^{(N+1)/2})} \displaystyle\sum\limits_{k=1}^{\infty} k^2 p^{k}.\label{eq:variance_upper_bound}
    \end{align}
    Applying Lemma~\ref{lemma:second_moment_geometric} in (\ref{eq:variance_upper_bound}) and simplifying, we get that $\Var[X] \le \frac{2 p (1+p)}{(1-p)^2 (1+p-2p^{(N+1)/2})}$.
\end{proof}
}
{The proof of Lemma~\ref{lemma:distribution_mean_variance} appears in the supplementary material.}

The next lemma will be used to show that our algorithm is differentially private with respect to single-user changes.
\begin{lemma}\label{lemma:classical_pointwise_DP}
 Let $w_1, w_2$ be two independent random variable sampled from the truncated discrete Laplace distribution $\D_{N,p}$ where $N$ is any positive odd integer and $p \in (0,1)$ is any real number, and let $z_1 = \frac{w_1}{k}$ and $z_2 = \frac{w_2}{k}$. For any $y^* \in \binom{\mathcal{Y}}{2m}_{ \gamma\textrm{-smooth}}$ and for all $x_1, x_2, x'_1 \in [0,1)$, if we denote $\tilde{x}_1 = \frac{\left \lfloor{x_1 k}\right \rfloor}{k}$, $\tilde{x}_2 = \frac{\left \lfloor{x_2 k}\right \rfloor}{k}$ and $\tilde{x}'_{1} = \frac{\left \lfloor{x'_1 k}\right \rfloor}{k}$, then
 \begin{equation}
     \Pr[E(\tilde{x}_1, \tilde{x}_2+z_2) = y^*] \le \frac{1+\gamma}{1-\gamma} \cdot p^{-k} \cdot \Pr[E(\tilde{x}'_1, \tilde{x}_2+z_2) = y^*],\label{eq:first_classical_DP_bound}
 \end{equation}
  \begin{equation}
     \Pr[E(\tilde{x}_1+z_1, \tilde{x}_2) = y^*] \le \frac{1+\gamma}{1-\gamma} \cdot p^{-k} \cdot \Pr[E(\tilde{x}'_1+z_1, \tilde{x}_2) = y^*].\label{eq:second_classical_DP_bound}
 \end{equation}
  and
  \begin{equation}
     \Pr[E(\tilde{x}_1+z_1, \tilde{x}_2+z_2) = y^*] \le \frac{1+\gamma}{1-\gamma} \cdot p^{-k} \cdot \Pr[E(\tilde{x}'_1+z_1, \tilde{x}_2+z_2) = y^*].\label{eq:third_classical_DP_bound}
 \end{equation}
 \end{lemma}

\splitter{
\begin{proof}[Proof of Lemma~\ref{lemma:classical_pointwise_DP}]
As in Lemma~\ref{lemma:noise_distribution_ratio}, we define the interval $I = \{-\frac{(N-1)}{2}, \dots, +\frac{(N-1)}{2}\}$. We define
\begin{equation}\label{eq:number_x_1_def}
    B_{y^*, x_1} := \text{Number of subsets } S \text{ of } \{1, \dots, 2m\} \text{ of size }m \text{ for which } \sum_{i\in S} y^*_i = \left \lfloor{x_1 k}\right \rfloor \mod N.
\end{equation}
We similarly define $B_{y^*, x'_1}$ and $B_{y^*, x_2}$ by replacing $x_1$ in~(\ref{eq:number_x_1_def}) by $x'_1$ and $x_2$ respectively.
\paragraph{Proof of Inequality~(\ref{eq:first_classical_DP_bound}).} By Lemma~\ref{lemma:two_player_set_probability}, we have that
    \begin{align}
         &\Pr[E(\tilde{x}_1, \tilde{x}_2+z_2) = y^*]\nonumber\\ 
         &=\frac{1}{|\mathcal{Y}|^{2 (m-1)}} \cdot \displaystyle\sum\limits_{\pi \in S_{2m}} \indicator_{E(\tilde{x}_1, \tilde{x}_2+z_2;\pi(y^*)_1,\dots,\pi(y^*)_{m-1}, \pi(y^*)_{m+1}, \dots, \pi(y^*)_{2m-1}) = \pi(y^*)}\nonumber\\ 
         &= \frac{(m!)^2}{|\mathcal{Y}|^{2 (m-1)}} \cdot B_{y^*, x_1} \cdot \Pr_{z_2 \sim \D_{N, p}}[z_2= (\sum_{i \in [2m]} y^*_i -\left \lfloor{x_1 k}\right \rfloor -\left \lfloor{x_2 k}\right \rfloor) \mod N]\\ 
         &= \frac{(m!)^2}{|\mathcal{Y}|^{2 (m-1)}} \cdot B_{y^*, x_1} \cdot \D_{N, p}[(\sum_{i \in [2m]} y^*_i -\left \lfloor{x_1 k}\right \rfloor -\left \lfloor{x_2 k}\right \rfloor) \mod I]. \label{eq:x_prob_2}
    \end{align}
    By Lemma~\ref{lemma:two_player_set_probability}, we also have that
    \begin{align}
         &\Pr[E(\tilde{x}'_1, \tilde{x}_2+z_2) = y^*]\nonumber\\ 
         &=\frac{1}{|\mathcal{Y}|^{2 (m-1)}} \cdot \displaystyle\sum\limits_{\pi \in S_{2m}} \indicator_{E(\tilde{x}'_1, \tilde{x}_2+z_2;\pi(y^*)_1,\dots,\pi(y^*)_{m-1}, \pi(y^*)_{m+1}, \dots, \pi(y^*)_{2m-1}) = \pi(y^*)}\nonumber\\ 
         &= \frac{(m!)^2}{|\mathcal{Y}|^{2 (m-1)}} \cdot B_{y^*, x'_1} \cdot \Pr_{z_2 \sim \D_{N, p}}[z_2 = (\sum_{i \in [2m]} y^*_i -\left \lfloor{x'_1 k}\right \rfloor -\left \lfloor{x_2 k}\right \rfloor) \mod N]\nonumber\\ &=\frac{(m!)^2}{|\mathcal{Y}|^{2 (m-1)}} \cdot B_{y^*, x'_1} \cdot \D_{N, p}[(\sum_{i \in [2m]} y^*_i -\left \lfloor{x'_1 k}\right \rfloor -\left \lfloor{x_2 k}\right \rfloor) \mod I].\label{eq:x_prime_prob}
    \end{align}
    Since $y^*$ is $\gamma\textrm{-smooth}$, Definition~\ref{def:gamma_smooth} implies that
    \begin{align}
        \frac{B_{y^*, x_1}}{B_{y^*, x'_1}} \le \frac{1+\gamma}{1-\gamma}.\label{eq:ratio_subsets_upper_bound}
    \end{align}
    Applying Lemma~\ref{lemma:noise_distribution_ratio} with $k = \sum_{i \in [2m]} y^*_i -\left \lfloor{x'_1 k}\right \rfloor -\left \lfloor{x_2 k}\right \rfloor$ and $t=\left \lfloor{x'_1 k}\right \rfloor-\left \lfloor{x_1 k}\right \rfloor$ and using the fact that $x_1, x'_1 \in [0,1)$ gives
    \begin{align}
        \frac{\D_{N, p}[(\sum_{i \in [2m]} y^*_i -\left \lfloor{x_1 k}\right \rfloor -\left \lfloor{x_2 k}\right \rfloor) \mod I]}{\D_{N, p}[(\sum_{i \in [2m]} y^*_i -\left \lfloor{x'_1 k}\right \rfloor -\left \lfloor{x_2 k}\right \rfloor) \mod I]} \le p^{-|\left \lfloor{x'_1 k}\right \rfloor-\left \lfloor{x_1 k}\right \rfloor|} \le p^{-k}.\label{eq:noise_ratio_upper_bound}
    \end{align}
    Dividing~(\ref{eq:x_prob_2}) by~(\ref{eq:x_prime_prob}) and using~(\ref{eq:ratio_subsets_upper_bound}) and~(\ref{eq:noise_ratio_upper_bound}), we get Inequality~(\ref{eq:first_classical_DP_bound}).\\ 
    \paragraph{Proof of Inequality~(\ref{eq:second_classical_DP_bound}).} We note that similarly to~(\ref{eq:x_prime_prob}) we have
    \begin{equation}
        \Pr[E(\tilde{x}_1+z_1, \tilde{x}_2) = y^*] = \frac{(m!)^2}{|\mathcal{Y}|^{2 (m-1)}} \cdot |B_{x_2}| \cdot \D_{N, p}[(\sum_{i \in [2m]} y^*_i -\left \lfloor{x_1 k}\right \rfloor -\left \lfloor{x_2 k}\right \rfloor) \mod I],\label{eq:second_x_prob}
    \end{equation}
    and
    \begin{equation}
        \Pr[E(\tilde{x}'_1+z_1, \tilde{x}_2) = y^*] = \frac{(m!)^2}{|\mathcal{Y}|^{2 (m-1)}} \cdot |B_{x_2}| \cdot \D_{N, p}[(\sum_{i \in [2m]} y^*_i -\left \lfloor{x'_1 k}\right \rfloor -\left \lfloor{x_2 k}\right \rfloor) \mod I],\label{eq:second_x_prime_prob}
    \end{equation}
    Dividing~(\ref{eq:second_x_prob}) by~(\ref{eq:second_x_prime_prob}) and using~(\ref{eq:noise_ratio_upper_bound}), we get Inequality~(\ref{eq:second_classical_DP_bound}).
    \paragraph{Proof of Inequality~(\ref{eq:third_classical_DP_bound}).}
    By averaging over $z_2$ and applying Inequality~(\ref{eq:second_classical_DP_bound}) with $\tilde{x}_2$ replaced by $\tilde{x}_2+z_2$ (for every fixed setting of $z_2$), we get Inequality~(\ref{eq:third_classical_DP_bound}).
\end{proof}
}
{The proof of Lemma~\ref{lemma:classical_pointwise_DP} appears in the supplementary material.}

\begin{lemma}\label{lemma:subset_DP_single_user_changes}
Let $N$ be a positive odd integer and $p \in (0,1)$ and $q \in (0,1]$ be real numbers. Let $b_1, \dots, b_n$ be iid random variables that are equal to $1$ with probability $q$ and to $0$ otherwise, let $w_1, \dots, w_n$ be iid random variables that are drawn from the truncated discrete Laplace distribution $\D_{N,p}$ independently of $b_1, \dots, b_n$, and let $z_i = \frac{b_i w_i}{k}$ for all $i \in [n]$. Then, for all $j \in [n]$, all $x_1, \dots, x_j, \dots, x_n, x'_j \in [0,1)$, if we denote $\tilde{x}_i = \frac{\left \lfloor{x_i k}\right \rfloor}{k}$ for all $i \in [n]$ and $\tilde{x}'_{j} = \frac{\left \lfloor{x'_j k}\right \rfloor}{k}$, then for all $S$, the following inequality holds
\begin{align}
    &\Pr[E(\tilde{x}_1 + z_1, \dots,  \tilde{x}_j+z_j, \dots, \tilde{x}_n + z_n) \in S]\nonumber\\ 
    &\le \frac{1+\gamma}{1-\gamma} \cdot \frac{p^{-k}}{1-e^{-qn}} \cdot \Pr[E(\tilde{x}_1+z_1, \dots, \tilde{x}'_j+z_j, \dots,\tilde{x}_n+z_n) \in S] + \eta + e^{-qn},\label{eq:subset_upper_bound_single_user_change}
\end{align}
for any $\gamma > \frac{6\sqrt{m}}{2^{2m}}$, $m\geq 4$ and $\eta = \frac{2m^2}{N} + \frac{18\sqrt{m}\,N^2}{\gamma^2 2^{2m}}$, and where the probabilities in~(\ref{eq:subset_upper_bound_single_user_change}) are over $z_1, \dots, z_n$ and the internal randomness of $E(\cdot)$.
\end{lemma}

\splitter{
\begin{proof}[Proof of Lemma~\ref{lemma:subset_DP_single_user_changes}]
Let $A$ denote the event that there exists at least one $i \in [n]$ for which $b_i = 1$. Then,
\begin{equation}\label{eq:lower_bound_at_least_one_noisy}
    \Pr[A] = 1-(1-q)^n \geq 1-e^{-qn},
\end{equation}
where the last inequality follows from the fact that $e^t \geq 1+t$ for any real number $t$. To prove~(\ref{eq:subset_upper_bound_single_user_change}), it suffices to show a similar inequality conditioned on the event $A$, i.e.,
\begin{align}
    &\Pr[E(\tilde{x}_1 + z_1, \dots,  \tilde{x}_j+z_j, \dots, \tilde{x}_n + z_n) \in S \mid A]\nonumber\\ 
    &\le \frac{1+\gamma}{1-\gamma} \cdot p^{-k} \cdot \Pr[E(\tilde{x}_1+z_1, \dots, \tilde{x}'_j+z_j, \dots,\tilde{x}_n+z_n) \in S \mid A] + \eta.\label{eq:conditional_subset_upper_bound_single_user_change}
\end{align}
To see this, denote by $\overline{A}$ the complement of the event $A$ and assume that~(\ref{eq:conditional_subset_upper_bound_single_user_change}) holds. Then,
\begin{align}
    &\Pr[E(\tilde{x}_1 + z_1, \dots,  \tilde{x}_j+z_j, \dots, \tilde{x}_n + z_n) \in S]\nonumber\\ 
    &= \Pr[A] \cdot \Pr[E(\tilde{x}_1 + z_1, \dots,  \tilde{x}_j+z_j, \dots, \tilde{x}_n + z_n) \in S \mid A]\nonumber\\ 
    &+ \Pr[\overline{A}] \cdot \Pr[E(\tilde{x}_1 + z_1, \dots,  \tilde{x}_j+z_j, \dots, \tilde{x}_n + z_n) \in S \mid \overline{A}]\nonumber\\ 
    &\le \Pr[E(\tilde{x}_1 + z_1, \dots,  \tilde{x}_j+z_j, \dots, \tilde{x}_n + z_n) \in S \mid A] + e^{-qn}\label{eq:tail_at_least_one_noisy_1}\\ 
    &\le \frac{1+\gamma}{1-\gamma} \cdot p^{-k} \cdot \Pr[E(\tilde{x}_1+z_1, \dots, \tilde{x}'_j+z_j, \dots,\tilde{x}_n+z_n) \in S \mid A] + \eta + e^{-qn}\label{eq:conditional_subset_upper_bound_app}\\ 
    &\le \frac{1+\gamma}{1-\gamma} \cdot p^{-k} \cdot \frac{\Pr[E(\tilde{x}_1+z_1, \dots, \tilde{x}'_j+z_j, \dots,\tilde{x}_n+z_n) \in S]}{\Pr[A]} + \eta + e^{-qn}\nonumber\\ 
    &\le \frac{1+\gamma}{1-\gamma} \cdot \frac{p^{-k}}{1-e^{-qn}} \cdot \Pr[E(\tilde{x}_1+z_1, \dots, \tilde{x}'_j+z_j, \dots,\tilde{x}_n+z_n) \in S] + \eta + e^{-qn}\label{eq:tail_at_least_one_noisy_2},
\end{align}
where~(\ref{eq:tail_at_least_one_noisy_1}) and~(\ref{eq:tail_at_least_one_noisy_2}) follow from~(\ref{eq:lower_bound_at_least_one_noisy}), and where~(\ref{eq:conditional_subset_upper_bound_app}) follows from the assumption that~(\ref{eq:conditional_subset_upper_bound_single_user_change}) holds. We thus turn to the proof of~(\ref{eq:conditional_subset_upper_bound_single_user_change}). Note that it suffices to prove this inequality for any fixed setting of $b_1, \dots, b_n$ satisfying the event $A$, i.e.,
\begin{align}
    &\Pr[E(\tilde{x}_1 + z_1, \dots,  \tilde{x}_j+z_j, \dots, \tilde{x}_n + z_n) \in S \mid b_1, \dots, b_n]\nonumber\\ 
    &\le \frac{1+\gamma}{1-\gamma} \cdot p^{-k} \cdot \Pr[E(\tilde{x}_1+z_1, \dots, \tilde{x}'_j+z_j, \dots,\tilde{x}_n+z_n) \in S \mid b_1, \dots, b_n] + \eta,\label{eq:point_conditional_subset_upper_bound_single_user_change}
\end{align}
and~(\ref{eq:conditional_subset_upper_bound_single_user_change}) would follows from~(\ref{eq:point_conditional_subset_upper_bound_single_user_change}) by averaging. Henceforth, we fix a setting of $b_1, \dots, b_n$ satisfying the event $A$. Without loss of generality, we assume that $j = 1$. If $b_j = 0$, then the event $A$ implies that there exists $j_2 \neq j$ such that $b_{j_2} = 1$. Without loss of generality, we assume that $j_2 = 2$. In order to show~(\ref{eq:point_conditional_subset_upper_bound_single_user_change}) for this setting of $b_1, \dots, b_n$, it suffices to show the same inequality where we also condition on any setting of $w_3, \dots, w_n$, i.e.,
\begin{align}
    &\Pr[E(\tilde{x}_1 + z_1, \dots,  \tilde{x}_j+z_j, \dots, \tilde{x}_n + z_n) \in S \mid b_1, \dots, b_n, w_3, \dots, w_n]\nonumber\\ 
    &\le \frac{1+\gamma}{1-\gamma} \cdot p^{-k} \cdot \Pr[E(\tilde{x}_1+z_1, \dots, \tilde{x}'_j+z_j, \dots,\tilde{x}_n+z_n) \in S \mid b_1, \dots, b_n, w_3, \dots, w_n] + \eta,\label{eq:with_w_point_conditional_subset_upper_bound_single_user_change}
\end{align}

Applying Lemma~\ref{lemma:swap_dp} with $j_1 = j =1$ and $j_2 = 2$ and with inputs $\tilde{x}'_3+z_3, \dots, \tilde{x}'_n+z_n$ for the non-selected players, we get that to prove~(\ref{eq:with_w_point_conditional_subset_upper_bound_single_user_change}), it suffices to show that for any set $T$, the following inequality holds
 \begin{equation}\label{eq:two_player_subset_upper_bound_single_user_change}
  \Pr[E(\tilde{x}_1 + z_1, \tilde{x}_2 + z_2) \in T] \leq \frac{1+\gamma}{1-\gamma} \cdot p^{-k} \cdot \Pr[E(\tilde{x}'_1 + z_1, \tilde{x}_2 + z_2) \in T] + \eta.
 \end{equation}
 We now prove~(\ref{eq:two_player_subset_upper_bound_single_user_change}):
  \begin{align}
      &\Pr[E(\tilde{x}_1 + z_1, \tilde{x}_2 + z_2) \in T]\nonumber\\ 
      &\leq \Pr\left[E(\tilde{x}_1 + z_1, \tilde{x}_2 + z_2) \not\in \smoothset\right] + \Pr\left[E(\tilde{x}_1 + z_1, \tilde{x}_2 + z_2) \in T \cap \smoothset\right] \nonumber\\
      &\leq \eta + \sum_{A\in T\cap\smoothset} \Pr[E(\tilde{x}_1 + z_1, \tilde{x}_2 + z_2)=A] \label{eq:averaging_nonsmooth}\\
      &\leq \eta + \sum_{A\in T\cap\smoothset} \frac{1+\gamma}{1-\gamma} \cdot p^{-k} \cdot \Pr[E(\tilde{x}'_1 + z_1, \tilde{x}_2 + z_2)=A] \label{eq:single_user_smooth_bound}\\
      &\leq \eta + \frac{1+\gamma}{1-\gamma} \cdot p^{-k} \cdot \Pr[E(\tilde{x}'_1 + z_1, \tilde{x}_2 + z_2) \in T], \nonumber
  \end{align}
  with $\eta = \frac{2m^2}{N} + \frac{18\sqrt{m}\,N^2}{\gamma^2 2^{2m}}$ and where~(\ref{eq:averaging_nonsmooth}) follows by averaging over all settings of $z_1, z_2$ and invoking Lemma~\ref{lemma:gamma_smooth}, and~(\ref{eq:single_user_smooth_bound}) follows from Lemma~\ref{lemma:classical_pointwise_DP} and the fact that at least one of $b_1, b_2$ is equal to $1$.
\end{proof}
}
{The proof of Lemma~\ref{lemma:subset_DP_single_user_changes} appears in the supplementary material.}
\splitter{
As a consequence, we obtain the following main theorem establishing differential privacy of Algorithm~\ref{alg:encoder} with respect to single-user changes in the shuffled model:

\mainsingledp*
}
{We are now ready to prove Theorem~\ref{theorem:main_classical_DP}.}

\begin{proof}[Proof of Theorem~\ref{theorem:main_classical_DP}]
In Algorithm~\ref{alg:encoder}, each user communicates at most $O(m \log{N})$ bits which are sent via $m$ messages.
By Lemma~\ref{lemma:subset_DP_single_user_changes}, Algorithm~\ref{alg:encoder} is $(\epsilon, \delta)$-differentially private with respect to single-user changes if $\frac{1+\gamma}{1-\gamma} \cdot \frac{p^{-k}}{1-e^{-qn}} \le e^{\epsilon}$, and $\frac{2m^2}{N} + \frac{18\sqrt{m}\,N^2}{\gamma^2 2^{2m}} + e^{-qn} \le \delta$, for any $\gamma > \frac{6\sqrt{m}}{2^{2m}}$ and $m\geq 4$. The error in our final estimate consists of two parts: the rounding error which is $O(n/k)$ in the worst case, and the error due to the added folded Discrete Laplace noise whose average absolute value is at most $O\bigg(\frac{\sqrt{qn}}{1-p}\bigg)$ (this follows from Lemma~\ref{lemma:distribution_mean_variance} along with the facts that the variance is additive for independent random variables, and that for any zero-mean random variable $X$, it is the case that $\E[|X|] \le \sqrt{\Var[X]}$). The theorem now follows by choosing $p = 1-\frac{\epsilon}{10k}$, $q=10\frac{\log(1/\delta)}{n}$, $m = 10\log\bigg(\frac{nk}{\epsilon \delta}\bigg)$, $\gamma = \frac{\epsilon}{10}$, $k = 10n$ and $N$ being the first odd integer larger than $3kn+\frac{10}{\delta} + \frac{10}{\epsilon}$.
\end{proof}

\subsection{Resilience Against Colluding Users}\label{sec:resilience}

In this section, we formalize the resilience of Algorithm~\ref{alg:encoder} against a very large fraction of the users colluding with with the server (thereby revealing their inputs and messages).

\begin{lemma}[Resilient privacy under sum-preserving changes]\label{lemma:resilient_privacy_sum_preserving}
Let $C \subseteq [n]$ denote the subset of colluding users. Then, for all $x_1, \dots, \dots, x_n$ and $x'_1, \dots, \dots, x'_n$ that are integer multiples of $1/k$ in the interval $[0,1)$ and that satisfy $\sum_{j \notin C} x_{j} = \sum_{j \notin C} x'_{j}$ and $x'_j = x_j$ for all $j \in C$, and for all subsets $S$, the following inequality holds
\begin{align}
    &\Pr[E(x_1, \dots, x_n) \in S \mid E(x_i) ~ \forall i \in C]\nonumber\\ 
    &\le \beta^{n-1} \cdot \Pr[E(x'_1, \dots,x'_n) \in S \mid E(x_i) ~ \forall i \in C] + \frac{(\beta^{n-1}-1)}{(\beta-1)} \cdot \eta,\label{eq:resilient_sum_preserving_bound}
\end{align}
for $\beta = \frac{1+\gamma}{1-\gamma}$, any $\gamma > \frac{6\sqrt{m}}{2^{2m}}$, $m\geq 4$ and $\eta = \frac{2m^2}{N} + \frac{18\sqrt{m}\,N^2}{\gamma^2 2^{2m}}$, and where the probabilities in~(\ref{eq:resilient_sum_preserving_bound}) are over the internal randomness of $E(\cdot)$.
\end{lemma}

\begin{lemma}[Resilient privacy under single-user changes]\label{lemma:resilient_privacy_single_user}
Let $N$ be a positive odd integer and $p \in (0,1)$ and $q \in (0,1]$ be real numbers. Let $C \subseteq [n]$ denote the subset of colluding users. Let $b_1, \dots, b_n$ be iid random variables that are equal to $1$ with probability $q$ and to $0$ otherwise, let $w_1, \dots, w_n$ be iid random variables that are drawn from the folded discrete Laplace distribution $\D_{N,p}$ independently of $b_1, \dots, b_n$, and let $z_i = \frac{b_i w_i}{k}$ for all $i \in [n]$. If $|C| \le 0.9 n$, then for all $j \notin C$, all $x_1, \dots, x_j, \dots, x_n, x'_j \in [0,1)$ and all subsets $S$, if we denote $\tilde{x}_i = \frac{\left \lfloor{x_i k}\right \rfloor}{k}$ for all $i \in [n]$ and $\tilde{x}'_{j} = \frac{\left \lfloor{x'_j k}\right \rfloor}{k}$, then
\begin{align}
    &\Pr[E(\tilde{x}_1+z_1, \dots, \tilde{x}_j+z_j, \dots,\tilde{x}_n+z_n) \in S \mid E(\tilde{x}_i+z_i) ~ \forall i \in C]\nonumber\\ 
    &\le \frac{1+\gamma}{1-\gamma} \cdot \frac{p^{-k}}{1-e^{-q(n-|C|)}} \cdot \Pr[E(\tilde{x}_1+z_1, \dots, \tilde{x}'_j+z_j, \dots,\tilde{x}_n+z_n) \in S \mid E(\tilde{x}_i+z_i) ~ \forall i \in C]\nonumber\\ 
    & + \eta + e^{-q(n-|C|)},\label{eq:resilient_single_user_bound}
\end{align}
for any $\gamma > \frac{6\sqrt{m}}{2^{2m}}$, $m\geq 4$ and $\eta = \frac{2m^2}{N} + \frac{18\sqrt{m}\,N^2}{\gamma^2 2^{2m}}$, and where the probabilities in~(\ref{eq:resilient_single_user_bound}) are over $z_1, \dots, z_n$ and the internal randomness of $E(\cdot)$.
\end{lemma}

\splitter{
\begin{proof}[Proof of Lemma~\ref{lemma:resilient_privacy_sum_preserving}]
We start by applying Lemma~\ref{lemma:swap_dp} in order to condition on the messages of all the colluding users. This allows us to reduce to the case where the messages of all users in $C$ are \emph{fixed} and where we would like to prove the differential privacy guarantee with respect to single-user changes on the inputs of the smaller subset $[n] \setminus C$ of (non-colluding) users. The rest of the proof follows along the same lines as the proof of Lemma~\ref{lemma:sum_swap_series} with any modification in the bounds.
\end{proof}

\begin{proof}[Proof of Lemma~\ref{lemma:resilient_privacy_single_user}]
We start by applying Lemma~\ref{lemma:swap_dp} in order to condition on the messages of all colluding users. This allows us to reduce to the case where the messages of all users in $C$ are \emph{fixed} and where we would like to prove differential privacy guarantees with respect to sum-preserving changes on the smaller subset $[n] \setminus C$ of (non-colluding) users. The rest of the proof follows along the same lines as the proof of Lemma~\ref{lemma:subset_DP_single_user_changes}. Note that that the tail probability term $e^{-qn}$ in~(\ref{eq:subset_upper_bound_single_user_change}) is replaced by the slightly larger quantity $e^{-q(n-|C|)}$ in~(\ref{eq:resilient_single_user_bound}) as the event $A$ in the proof of Lemma~\ref{lemma:subset_DP_single_user_changes} has now to be defined over the smaller set $[n] \setminus C$ of non-colluding users (and consequently the bounds in~(\ref{eq:lower_bound_at_least_one_noisy}) and~(\ref{eq:tail_at_least_one_noisy_2}) are modified similarly).
\end{proof}
}
{The proofs of Lemmas~\ref{lemma:resilient_privacy_sum_preserving} and~\ref{lemma:resilient_privacy_single_user} are given in the supplementary material.}

\section{Conclusion and Open Problems}\label{sec:conclusion}
Our work provides further evidence that the shuffled model of differential privacy~\cite{bittau17,cheu19} is a fertile ''middle ground'' between local differential privacy and general multi-party computations, combining the scalability of local DP with the high utility and privacy of MPC.
This makes it more feasible to design scalable machine learning systems in a federated setting.

The main open problem that we leave is how many messages $m$ are necessary to achieve differential privacy without a cost of $n^{\Omega(1)}$ in error or communication.
It is shown in~\cite{balle19} that $m=1$ is not enough, but we cannot rule out that $m = O(\log_n k)$ suffices to achieve error $1/k$ under sum-preserving changes, using our protocol unchanged.
Another issue is that our current protocol fails to provide privacy with some small probability, for example if all random numbers chosen by the encoder happen to be zero.
The question is whether the error probability can be eliminated by somehow changing the protocol, achieving pure differential privacy.

\bibliographystyle{abbrv}
\bibliography{main}

\newpage
\end{document}